\documentclass{article}

\usepackage[margin=0.9in]{geometry}
\usepackage{natbib}

\usepackage[utf8]{inputenc} %
\usepackage[T1]{fontenc}    %
\usepackage{hyperref}       %
\usepackage{url}            %
\usepackage{booktabs}       %
\usepackage{amsfonts}       %
\usepackage{nicefrac}       %
\usepackage{microtype}      %
\usepackage{xcolor}         %

\usepackage{amsmath}
\usepackage{amsthm}
\usepackage{amssymb}
\usepackage{tikz}
\usetikzlibrary{trees, positioning, shapes, arrows.meta, calc}

\newtheorem{definition}{Definition}
\newtheorem{remark}{Remark}
\newtheorem{example}{Example}

\usepackage[algo2e,ruled]{algorithm2e}
\usepackage{algorithm}
\usepackage{algpseudocode}
\usepackage{xcolor}

\usepackage{thmtools}
\usepackage{thm-restate}

\newcommand{\R}{\mathbb{R}} 
\newcommand{\N}{\mathbb{N}} 
\newcommand{\Z}{\mathbb{Z}} 
\newcommand{\calL}{\mathcal{L}} 
\newcommand{\calX}{\mathcal{X}} 
\newcommand{\calG}{\mathcal{G}}
\newcommand{\calY}{\mathcal{Y}}

\newcommand{\calR}{\mathcal{R}}

\title{Mistake-Bounded Language Generation}

\author{%
  Jon Kleinberg\thanks{Authors listed in alphabetical order.}%
    \\
  Departments of Computer Science \\and Information Science\\
  Cornell University\\
  Ithaca, NY\\
  \and
  Charlotte Peale \\
  Department of Computer Science\\
  Stanford University \\
  Stanford, CA \\
  \and
  Omer Reingold \\
  Department of Computer Science\\
  Stanford University \\
  Stanford, CA \\
}

\begin{document}

\maketitle

\begin{abstract}
  We investigate the learning task of language generation in the limit, but shift focus from the traditional time-of-last-mistake metric of a generator's success to a new notion of \emph{mistake-bounded generation}. While existing results for language generation in the limit focus on guaranteeing eventual consistency, they are blind to the cumulative error incurred during the learning process. We address this by shifting the goal to minimizing the total number of invalid elements output by a generation algorithm. We establish a formal reduction to the Learning from Correct Demonstrations framework of Joshi et al., enabling a general recipe for deriving mistake bounds via weighted update rules. For finite classes, we provide an algorithm that simultaneously achieves an optimal last-mistake time of $\mathsf{Cdim}(\mathcal{L})$ and a mistake bound of $\lfloor \log_2 |\mathcal{L}| \rfloor$, whereas for the non-uniform setting of countably infinite streams of languages, we prove a fundamental trade-off: achieving logarithmic mistakes $O(\log i)$ necessarily precludes convergence guarantees established in prior work. Finally, we show that our framework can be extended to accommodate noisy adversaries and guarantee mistake bounds that scale with the adversary's suboptimality.

\end{abstract}

\section{Introduction}
The capability of Large Language Models (LLMs) to generate coherent text has sparked a renewed interest in the theoretical underpinnings of language generation. While early learning-theoretic work such as Gold and Angluin's work on language identification in the limit \citep{gold1967language, angluin1980inductive} focused on the difficult task of identifying a language or grammar from positive samples, recent work by \cite{kleinberg2024language} hinges on the observation that a successful large language model need not fully identify its language of interest, but simply understand enough to generate new, valid examples from the language. 

This idea has been explored through the model of \emph{language generation in the limit}, first proposed by \cite{kleinberg2024language}. In this setting, a generator plays an online game against an adversary. At each step, the generator must output a new element $\hat x_t$ that has not yet been outputted by the adversary, and then the adversary reveals a new element $x_t$ from a target language $L^*$, known to come from some class of languages $\calL$. The generator succeeds if it eventually converges to outputting only elements from $L^*$. Subsequent work by \cite{li2024generation} further refined what it means to \emph{eventually converge}, dividing the space of language generation requirements into three tiers, listed in order of increasing strength: language generation in the limit, non-uniform generation, and uniform generation. 

Existing work in all three of these settings quantifies a generator's success through the \emph{time-of-last-mistake}, i.e. the timestep after which the generator always outputs from the target language $L^*$. While this metric guarantees eventual consistency, it is blind to the cost of learning incurred along the way. A generator could theoretically hallucinate invalid outputs for a million steps before converging, or could generate perfectly for a long time before making a single mistake after a million steps. In both cases, the ``time of last mistake'' might be identical, yet the practical utility of these two generators is vastly different. The following example shows that this latter case is indeed possible: 

\begin{example}[Separation between time-of-last mistake and number of mistakes.]\label{ex:venn}
    Consider a language class composed of two languages $\calL = \{L_1, L_2\}$, where $|L_1 \cap L_2| = n < \infty$. Consider the adversary that enumerates the elements falling in $L_1 \cap L_2$ for the first $n$ timesteps. At timestep $t_{n + 1}$, there are no unseen elements remaining in the intersection, and thus the generator is forced to output an element in $L_1 \setminus L_2$ or $L_2 \setminus L_1$. In either case, the adversary can force a mistake by observing the generator's choice, and then enumerating elements from the other language. Thus, the adversary can force the learner's time of last mistake to be as large as $n + 1$. 

    In contrast, if the generator also enumerates $L_1 \cap L_2$ while there are still unseen elements in the intersection, the generator is guaranteed to only make a single mistake at time $t = n + 1$ when it must play the first point outside of the intersection. 
\end{example}

Clearly the total number of mistakes is at most the time of the last mistake (since the generator outputs at most one element per time step), but Example~\ref{ex:venn} shows that the total number of mistakes can be arbitrarily lower: in the example, the time of the last mistake can be arbitrarily high depending on the choice of language class, but the number of mistakes is always bounded by 1. The fact that these two performance metrics operate in different ways raises a fundamental set of questions analogous to the study of mistake bounds in online classification~\citep{littlestone1988learning}:

\begin{quote}
    \textbf{Q1.} \textit{Are there upper bounds on the total number of mistakes a generator makes for general language classes that are dramatically better than the trivial upper bound inherited from the time of the last mistake?} 
    \\[5pt]
    \textbf{Q2.} \textit{Does minimizing the total number of mistakes a generator makes require sacrificing optimal guarantees on the time when it makes its last mistake?}
\end{quote}

We investigate this question by studying the interplay between the two metrics of last-mistake-time and total mistakes. While bounding mistakes is valuable on its own, an ideal generator would provide a dual guarantee: minimizing the total volume of errors while also ensuring that the generator converges to the true language as quickly as possible.

\subsection{Our Contributions}\label{sec:contributions}
In this work, we introduce the framework of \emph{Mistake-Bounded Language Generation}. We demonstrate that while the total number of mistakes is bounded by the time-of-last-mistake, it can be arbitrarily smaller. We provide algorithms that minimize this new metric and prove fundamental trade-offs between minimizing errors and minimizing convergence time.

Our primary contributions are:

\begin{enumerate}
    \item \textbf{A General Recipe for Mistake-Bounded Generation via Learning from Demonstrations:} We prove bounds on the total number of mistakes for language generation by first establishing a formal connection between language generation and the setting of Learning from Correct Demonstrations (LfD) \citep{joshi2025learning}.
    Leveraging this reduction, we adapt their techniques to provide mistake bounds in the language generation setting. Crucially, while the original LfD technique is restricted to finite classes, we extend the analysis to handle \emph{countably infinite} language streams through the use of a growth function and non-uniform prior over languages.

    \item \textbf{Optimal Bounds for Finite Classes:} For finite language classes, we prove an upper bound of $\min\{\log_2 |\mathcal{L}|, \mathsf{Cdim}(\mathcal{L})\}$ on the total mistakes. This directly addresses (Q1) above for the finite class case: While the closure dimension has been shown to tightly govern the optimal time-of-last-mistake for finite classes, our bound shows that we can do dramatically better in terms of total mistakes when $\log |\calL|$ is much smaller than $\mathsf{Cdim}(\calL)$. Furthermore, in response to (Q2) as posed above, we present a hybrid algorithm that \emph{simultaneously} achieves this mistake bound and the optimal time-of-last-mistake, proving that one need not sacrifice convergence speed to minimize errors in the finite class setting.

    \item \textbf{Non-Uniform Bounds and Trade-offs:} In Section \ref{sec:infinite-classes}, we apply our general recipe to infinite language classes $\calL = \{L_1, L_2, \dots\}$, deriving a non-uniform mistake bound of $O(\log i)$ for the $i$-th language. This provides a positive answer to (Q1) for the non-uniform setting, but reveals a more nuanced answer to (Q2) than the finite case: here, minimizing mistakes \emph{does} require sacrificing bounds on the time-of-last-mistake. We formalize this via a fundamental trade-off (Theorem~\ref{thm:non-uniform-tradeoff}), proving that any algorithm achieving the non-uniform convergence time guarantees of \citep{charikar2024exploring} must suffer a number of mistakes \emph{linear} in $i$, whereas our algorithm achieves a logarithmic bound by allowing for a slower convergence time.

    \item \textbf{Robustness to Noise:} Finally, we show that our reduction to the LfD setting naturally handles adversarial noise. We give mistake bounds for finite and countably infinite language classes in Lemmas~\ref{lem:noisy-finite} and \ref{lem:noisy-infinite}, respectively, and contrast these results with prior work on last-mistake-time guarantees in the presence of noise.
\end{enumerate}

\section{Related Work}

Building on \cite{kleinberg2024language}, subsequent work has advanced our understanding of language generation in the limit and investigated various extensions of the model. \cite{li2024generation} introduce and study the stronger models of non-uniform and uniform generation, and give characterizations of when these stronger generation goals are possible. \cite{charikar2024exploring} also give a simple algorithm for non-uniform generation that matches and in some cases improves the guarantees of \cite{li2024generation}, while exploring a number of other variants to the language generation model. 

Other follow-up works have sought to further understand which classes are and are not generatable. \cite{charikar2025pareto} study the pareto-frontier of non-uniform generation guarantees, while \cite{hanneke2025union} study whether finite unions of generatable classes are guaranteed to satisfy generatability, and answer this in the negative. \cite{anastasopoulos2026safe} study a variant termed \emph{safe} language generation, in which a generator must output elements from the target language while avoiding a subset of harmful elements. \cite{karbasi2025possibility} consider the problem of auditing a generator to determine whether it is hallucinating. \cite{arenas2025language} study the problem of language generation through a computational lens, and show barriers to computationally-efficient generation in various settings.

A significant line of works have considered the possibilities and impossibilities around \emph{language generation with breadth}, in which a generator must not only output consistent elements, but guarantee eventual ``coverage'' of the target language according to various definitions \cite{kalavasis2024characterizations, kalavasis2025limits, charikar2024exploring, peale2025representative, kleinberg2025density, kleinberg2025language}.

Finally, a last variant that has received significant study is language generation in the presence of \emph{noisy} adversaries. We overview these works in Section~\ref{sec:noisy-extension}.

\section{Mistake-Bounded Language Generation}
We begin by overviewing the learning task of language generation in the limit, introduced by \cite{kleinberg2024language}. We will focus on the settings of uniform and non-uniform language generation by \cite{li2024generation}. We then define our new notion of \emph{mistake-bounded generation}, and how it fits into this existing framework. 

Given a countably infinite universe $\calX$, a language generation in the limit problem is defined with respect to a (potentially infinite) language class $\calL$, where each language $L \in \calL$ specifies a particular infinite subset of the universe $\calX$. When $\calL$ is a countably infinite class, we assume that it is provided to us in the form of an infinite stream of languages $L_1, L_2, ....$. 

The learning task of language generation in the limit is formulated as an online game between a generation algorithm and an adversary. At each timestep $t$, the generator outputs an element $\hat x_t \in \calX$, and then the adversary reveals an element $x_t \in \calX$ (See Remark~\ref{rem:expositional-choices} for a note on order-of-play)%
. At each timestep, the generator must operate under the constraint that $\hat x_t$ cannot be an element that has been previously outputted by the adversary, e.g. $\hat x_t \not\in \{x_1, ..., x_{t - 1}\}$.  %

We make the following assumptions about the adversary. First, we assume that the adversary's outputs realize some language $L^* \in \calL$, i.e. there exists an $L^* \in \calL$ such that $x_t \in L^*$ for all timesteps $t \in \mathbb{N}$. Second, we assume that each of the adversary's outputs are unique, i.e. $x_t \not\in \{x_1, ..., x_{t - 1}\}$. See Remark~\ref{rem:expositional-choices} for further discussion about this assumption. 

At each timestep, the generator gets no information about the identity of the true language $L^*$ beyond the transcript of elements outputted by the adversary, $x_1, ..., x_t$, which we will denote as $x_{1:t}$. Nevertheless, the generator is judged by its ability to eventually generate elements from $L^*$. The particular notion of \emph{eventually} is what separates the three tiers of language generation, listed in order of strength: language generation in the limit, non-uniform generation, and uniform generation. In this paper, we focus on the latter two, and define them below. 

\begin{definition}[Non-Uniform Language Generation]
    A generation algorithm $\mathcal{G}$ satisfies \\non-uniform generation for a language class $\calL$ if for every $L \in \calL$, there exists a finite $t(\mathcal{G}, L) < \infty$ such that, when given as input any stream of unique adversarial outputs $x_1, x_2, ...$ realized by $L$, for all timesteps $t \geq t(\mathcal{G}, L)$, the algorithm is guaranteed to generate from $L$, i.e. $\hat x_{t} \in L \setminus \{x_{1:t-1}\}$. 
\end{definition}

In non-uniform generation, the generator is guaranteed to start generating elements consistent with the true language at some finite timestep. However, the particular timestep where this is guaranteed to happen can be different for each language. This means that it's possible that even though $t(\mathcal{G}, L)$ is finite for each language, $\max_{L \in \calL} t(\mathcal{G}, L)$ is infinite. Uniform generation strengthens the requirement to hold over all languages. In other words, the value of this max expression must be finite:

\begin{definition}[Uniform Language Generation]
    A generation algorithm $\mathcal{G}$ satisfies uniform generation for a language class $\calL$ if there exists a time $T \in \mathbb{N}$ such that for adversarial outputs $x_1, x_2, ...$ realized by a language $L \in \calL$, for all timesteps $t \geq T$ the generator is guaranteed to generate from $L$, i.e. $\hat x_t \in L \setminus \{x_{1:t-1}\}$. 
\end{definition}

\begin{remark}\label{rem:expositional-choices}
    We note two deviations from standard generation setups, both made purely for expositional ease. First, we reverse the standard order of play so that the generator moves first. This choice does not materially affect theoretical guarantees, as any algorithm designed for one setting can be adapted to the other, shifting the total mistake bound or time-of-last-mistake by at most 1. Second, we assume the adversary outputs only unique elements. Existing definitions already measure eventual consistency times ($t(\mathcal{G}, L)$ and $T$) against the number of \emph{unique} items seen, preventing the adversary from stalling learning by indefinitely repeating elements. Our setup simply ignores steps where a duplicate is output.
\end{remark}

\subsection{A new success metric: number of mistakes}

One can view the guarantees of uniform and non-uniform generation as providing a bound on the \emph{time-of-last-mistake}, i.e., the last timestep $t$ where the generator outputs an element that is not part of the true language. As illustrated in Example~\ref{ex:venn}, the \emph{total number of mistakes} made by the generator can be significantly smaller than the time of last mistake.

In this work, we aim to minimize this cumulative error. While for finite classes we simply aim to bound the total mistakes in terms of $|\calL|$, for the countably infinite classes considered in Section~\ref{sec:infinite-classes}, we require a definition that allows the mistake bound to scale with the complexity (position) of the target language in the stream:

\begin{definition}[Non-Uniform Mistake Bound]
    Given a countably infinite language class equipped with a fixed enumeration $\calL = \{L_1, L_2, ...\}$, a generator $\calG$ and a target language $L_i$, $\calG$'s \emph{non-uniform mistake bound} with respect to $L_i$ is defined as the worst-case number of mistakes over all possible streams of unique elements from $L_i$:
    \[M_{\mathrm{non\text{-}unif}}(\calG, i) := \sup_{x_1, x_2, \dots \in L_i} \sum_{t = 0}^{\infty} \mathbf{1}\big[\calG(x_{1:t}) \notin L_i\big].\]
\end{definition}

Thus, the goal for infinite language classes is to provide a finite bound on $M_{\mathrm{non\text{-}unif}}(\calG, i)$ for every language $L_i \in \calL$ .

\section{A General Recipe for Mistake-Bounded Generation}\label{sec:general-recipe}

While prior algorithms for language generation in the limit rely on keeping track of which languages are still consistent with the adversarial stream, bounding the total number of mistakes requires a refined approach that not only keeps track of which languages are consistent, but prioritizes generating from languages on which the generator has already made many mistakes.

Instead, we adopt a \emph{weighted} approach derived from a reduction to the "Learning from Correct Demonstrations" (LfD) framework introduced by \cite{joshi2025learning}. In Appendix~\ref{sec:lfd-inf-rewards}, we show how the problem of language generation can be mapped to the LfD setting. 

Our core algorithmic strategy builds upon the multiplicative weights update technique proposed by \cite{joshi2025learning} for the LfD setting. However, a direct application of their result is insufficient due to the fact that their result only implies mistake bound for finite language classes. The setting of language generation, particularly non-uniform generation, inherently requires handling \emph{countably infinite} streams of languages.

To bridge this gap, we introduce Algorithm~\ref{alg:non-uniform-mistake-bound}, which extends the ideas of \cite{joshi2025learning} to infinite language classes. This extension requires two new components not present in the original LfD setting:
\begin{enumerate}
    \item A \textbf{prior weight distribution} $w_0: \mathbb{N} \to \mathbb{R}_{\geq 0}$, representing our initial belief in the likelihood of each language.
    \item A \textbf{growth function} $f: \mathbb{N} \to \mathbb{N}$, which governs the ``capacity'' of the learner by determining how many languages are actively considered at timestep $t$. This growth function allows the algorithm to handle infinite streams with finite computational resources.
\end{enumerate}

\begin{remark}
    While we present this algorithm in the context of language generation, this extension also yields a novel result for the original LfD setting: it provides the first mistake bounds for learning from infinite streams of reward functions (see Appendix~\ref{sec:lfd-inf-rewards}).
\end{remark}

We show that our algorithm yields the following mistake bound guarantee:

\begin{restatable}{thm}{generalrecipe}\label{thm:general-recipe}
   For any countably infinite language class equipped with a fixed enumeration $\calL = \{L_1, L_2, ...\}$, associated weights $w_0: \mathbb{N} \rightarrow \R_{\geq 0}$ such that $\sum_{i = 1}^{\infty}w_0(i) \leq W < \infty$, and non-decreasing growth function $f:\N \rightarrow \N$, the generations outputted by Algorithm~\ref{alg:non-uniform-mistake-bound} guarantee a non-uniform mistake bound of $M(\mathcal{G}, L_i) \leq f^{-1}(i) + \lfloor \log_2(W/w_0(i))\rfloor$ for each $i \in \N$, where if there exists a $t'$ such that $f(t') \geq i$, $f^{-1}(i)$ is the largest value $t$ such that $f(t) < i$, or 0 if every for every $t'\in \mathbb{N}$ $f(t')$ is at least $i$, and otherwise $f^{-1}(i) = \infty$. 
\end{restatable}

The proof (provided in Appendix \ref{sec:general-recipe-pf}) relies on a potential function argument tracking the total weight of the system. Intuitively, every mistake made while $L_i$ is the target causes $w_t(i)$ to double. However, we also show that the actions taken by the generator ensure that the total weight of the system is bounded, thus preventing $w_t(i)$ from doubling indefinitely.

In the following sections, we instantiate this recipe with different weights $w_0$ and growth functions $f$ to derive mistake bounds for both finite (Section~\ref{sec:finite-classes}) and infinite (Section~\ref{sec:infinite-classes}) classes.

\begin{algorithm}[H] 
    \DontPrintSemicolon
    \SetAlgoNoLine
    
    \caption{Non-Uniform Mistake-Bounded Generation algorithm for a provided initial weighting of languages and growth function.}
    \label{alg:non-uniform-mistake-bound}
    
    \KwIn{Universe $\calX$, class $\mathcal{L} = \{L_1, L_2, \dots\}$, weights $w_0:\mathbb{N} \rightarrow \mathbb{R}_{\geq 0}$, growth function $f: \mathbb{N} \rightarrow \mathbb{N}$}
    
    \For{$t = 1, \dots, \infty$}{
        Output $\hat x_t := \arg\max_{x \in \mathcal{X} \setminus x_{1:t-1}}\sum_{i = 1}^{f(t)}w_{t - 1}(i)\mathbf{1}[x \in L_i]$\;\\
        Observe adversary's $x_t$\;\\
        Update weights for $i \in [f(t)]$:\\
        \[w_t(i) = \begin{cases}
            0 & x_t \notin L_i\\
            w_{t - 1}(i) & x_t, \hat{x}_t \in L_i\\
            2w_{t - 1}(i) & x_t \in L_i, \hat{x}_t \notin L_i
        \end{cases}\]\;
        
        Initialize weights for $i \in \{f(t) + 1, \dots, f(t + 1)\}$:\\
        \[w_t(i) = \begin{cases}
            w_0(i) & x_{1:t} \subseteq L_i\\
            0 & \text{otherwise}
        \end{cases}\]\;
    }
\end{algorithm}

\section{Generating from Finite Language Classes with $\log_2 |\calL|$ mistakes}\label{sec:finite-classes}
We begin by focusing on the special case of finite language classes. Existing results by \cite{li2024generation} and \cite{kleinberg2024language} have shown that any finite language class is uniformly generatable. Moreover, they prove that the optimal \emph{time-of-last-mistake} is tightly characterized by the \emph{closure dimension}.

\begin{definition}
    The \emph{closure dimension} of a language class $\calL$ is the size of the largest finite intersection of a sub-collection of languages from $\calL$:
    \[\mathsf{Cdim}(\calL) = \max \{ \left|\bigcap_{L \in \calL'} L\right| : \calL' \subseteq \calL, \left|\bigcap_{L \in \calL'} L\right| <\infty\}\]
\end{definition}

While \cite{li2024generation} provide an algorithm that guarantees that the last mistake happens by time $t = \mathsf{Cdim}(\calL)$, they also show that this bound is tight: there exist adversarial streams that force the generator to make a mistake at $t = \mathsf{Cdim}(\calL)$. However, this metric does not consider the total number of mistakes made on the way to the time of last mistake.

\subsection{A Hybrid Algorithm: Minimizing Mistakes and Convergence Time}
We now show that by instantiating our general recipe for mistake-bounded language generation (Algorithm~\ref{alg:non-uniform-mistake-bound}) with uniform weights, we can achieve a mistake bound that can be arbitrarily smaller than the closure dimension. Furthermore, we show that our algorithm's decisions coincide with critical decision points in the uniform generation algorithm of \cite{li2024generation}, thus satisfying \emph{both} optimality criteria simultaneously.

Let $\calL$ be a finite class of size $N$. We instantiate Algorithm~\ref{alg:non-uniform-mistake-bound} with uniform weights $w_0(i) = 1$ for all $L_i \in \calL$ and a static growth function that considers all languages at every point in time: $f(t) = N$ for all $t$. While this algorithm alone only provides a mistake bound, we show that its actions also satisfy a critical property: whenever the intersection of all consistent languages has infinite size, it plays an unseen element from this intersection. This is exactly the property required by the uniform generation algorithm of \cite{li2024generation}, and thus we can show that our algorithm also enjoys a last-mistake-time guarantee of $\mathsf{Cdim}(\calL)$. 

This analysis yields the following dual guarantee:

\begin{restatable}{thm}{finitehybridguarantee}\label{thm:generating-finite-class}
    For any finite class $\calL$, there exists a generation algorithm that uniformly generates from $\calL$ and guarantees that for any target $L^* \in \calL$ and adversarial sequence consistent with $L^*$:
    \begin{enumerate}
        \item The total number of mistakes is at most $\min\{\lfloor \log_2 |\calL|\rfloor, \mathsf{Cdim}(\calL)\}$.
        \item The time of last mistake is at most $\mathsf{Cdim}(\calL)$.
    \end{enumerate}
\end{restatable}

This result demonstrates a fundamental separation between our two metrics of total mistakes and time-of-last-mistake. It also recovers the intuition of Example~\ref{ex:venn}, in which the time until the last mistake depended on the closure dimension, $|L_1 \cap L_2|$, but there was a strategy that preserved this time-of-last-mistake guarantee while only making $\log_2|\calL| = 1$ mistake. 

We prove Theorem~\ref{thm:generating-finite-class} by bringing together two algorithms with different guarantees. In particular, we consider the uniform generation algorithm proposed by \cite{li2024generation} that guarantees a last-mistake-time upper-bounded by $\mathsf{Cdim}(\calL)$, but not necessarily any bound on the total number of mistakes, and the algorithm for finite language classes that we obtain via a reduction to \cite{joshi2025learning}, which upper-bounds the number of mistakes by $\log_2 |\calL|$, but gives no guarantee on the last mistake time. 

We show that our mistake-bounded algorithm can be viewed as a \emph{sharpened} version of the uniform generation algorithm of \cite{li2024generation}, which coincides with the uniform generation algorithm at critical points, but makes smart decisions when there is leeway between which points to choose to still ensure the time-of-last-mistake guarantee. 

We are thus able to combine the guarantees of both algorithms, and guarantee that the time-of-last-mistake is bounded by $\mathsf{Cdim}(\calL)$, while simultaneously guaranteeing that the total number of mistakes is at most $\log_2 |\calL|$. Because the time-of-last-mistake is also an upper bound on the total number of mistakes, we can improve the total mistake bound to $\min\{\log_2|\calL|, \mathsf{Cdim}(\calL)\}$ when the closure dimension is small.

\begin{algorithm}[H] 
    \DontPrintSemicolon
    \SetAlgoNoLine
    
    \caption{Uniform Generation Algorithm proposed by \cite{li2024generation, kleinberg2024language}}
    \label{alg:uniform-generation-algorithm}
    
    \KwIn{Universe $\calX$, finite class $\mathcal{L}$}
    
    \For{$t = 1, \dots, \infty$}{
        Compute $C = \{L \in \calL: x_{1:t - 1} \subseteq L\}$.\;\\
        \textcolor{red}{If $|\bigcap_{L\in C}L \setminus x_{1:t-1}| > 0$, output $\hat x_t \in \bigcap_{L\in C}L \setminus x_{1:t-1}$. \quad (1)}\\
        Otherwise, output an arbitrary $\hat{x_t} \in \calX \setminus x_{1:t-1}$.
    }
\end{algorithm}

\begin{restatable}[Uniform Generation Guarantee]{thm}{unifgen}
    \citep[((2.2), Lemma 2.3, resp.)]{kleinberg2024language, li2024generation}
    For any finite class $\calL$ and adversarial sequence of unique outputs consistent with $L \in \calL$, the generator specified by Algorithm~\ref{alg:uniform-generation-algorithm} has a last mistake time of at most $\mathsf{Cdim}(\calL)$. 
\end{restatable}

We highlight line (1) of Algorithm~\ref{alg:uniform-generation-algorithm} as the key step necessary to obtain the last mistake time guarantee. Note that all other decisions made by the algorithm are completely arbitrary. This idea will be used to prove Theorem~\ref{thm:generating-finite-class}. In particular, we will give a mistake-bounded algorithm, and then show that it is consistent with the decisions of Algorithm~\ref{alg:uniform-generation-algorithm} and thus also guarantees a last-mistake-time of $\mathsf{Cdim}(\calL)$. 

We are now ready to prove Theorem~\ref{thm:generating-finite-class}.

\begin{proof}[Proof of Theorem~\ref{thm:generating-finite-class}]
    Denote $\calL = L_1, ..., L_k$, where $k = |\calL|$. We instantiate Algorithm~\ref{alg:non-uniform-mistake-bound} with $w_0(i)= 1$ for all $i \in [k]$, and $f(t) = k$ for all $t \geq 1$. Thus, note that $f^{-1}(i) = 0$ for all $i \in [k]$. 

    By Theorem~\ref{thm:general-recipe}, under this instantiation, the algorithm is guaranteed to make at most $\lfloor \log_2 k\rfloor$ mistakes on any adversarial stream consistent with some $L \in \calL$. 

    We now consider the time of last mistake bound. Due to the guarantees of Algorithm~\ref{alg:uniform-generation-algorithm}, it suffices to show that at any timestep $t$ such that $C = \{L \in \calL : x_{1:t-1} \subseteq L\}$ and $|\bigcap_{L \in C}L \setminus x_{1:t - 1}| > 0$, the $\hat{x}_t$ output by our mistake-bounded algorithm, Algorithm~\ref{alg:non-uniform-mistake-bound}, with the specified weight and growth functions, satisfies $\hat x_t \in \bigcap_{L \in C}L \setminus x_{1:t - 1}$. 

    We note that by definition of our algorithm, for any $L_i \not\in C$, we have $w_{t - 1}(i) = 0$, and for all $L_j \in C$, $w_{t - 1}(j) > 0$. 

    Thus, for any $x$ in the set $\bigcap_{L \in C}L \setminus x_{1:t - 1}$, the sum of the weights of its consistent languages achieves the maximum possible value, as it is the sum of all languages with non-zero weight, while any other $x$ outside of the intersection must have strictly less total weight, because it is not included in one of the languages with non-zero weight. 

    We conclude that whenever $\left|\bigcap_{L \in C}L \setminus x_{1:t - 1}\right| > 0$, the $\hat x_t$ output by Algorithm~\ref{alg:non-uniform-mistake-bound} defined as 
    \[\hat x_t = {\arg\max}_{x \in \calX\setminus x_{1:t-1}}\sum_{i = 1}^k w_{t - 1}(i)\mathbf{1}[x \in L_i]\]
    is guaranteed to satisfy $\hat x_t \in \bigcap_{L \in C}L \setminus x_{1:t - 1}.$

    This means that the outputs of Algorithm~\ref{alg:non-uniform-mistake-bound} are consistent with the decisions of Algorithm~\ref{alg:uniform-generation-algorithm}, and so we conclude that the algorithm not only provides a mistake bound, but also guarantees that the time of last mistake is at most $\mathsf{Cdim}(\calL)$. Because the number of mistakes can be at most the time of last mistake, this further strengthens the mistake bound to $\min\{\lfloor \log_2 k\rfloor, \mathsf{Cdim}(\calL)\}$. 
\end{proof}

\subsection{Lower Bounds}

We note that while our algorithm's last mistake bound has been shown to be tight for every class, the mistake bound is not necessarily optimal in a per-class sense. However, the following lower bound demonstrates that this is the best we can hope to achieve from a bound that depends solely on the cardinality of $\calL$. We prove this via a lower bound construction based on Littlestone trees (see Figure~\ref{fig:littlestone-tree}).

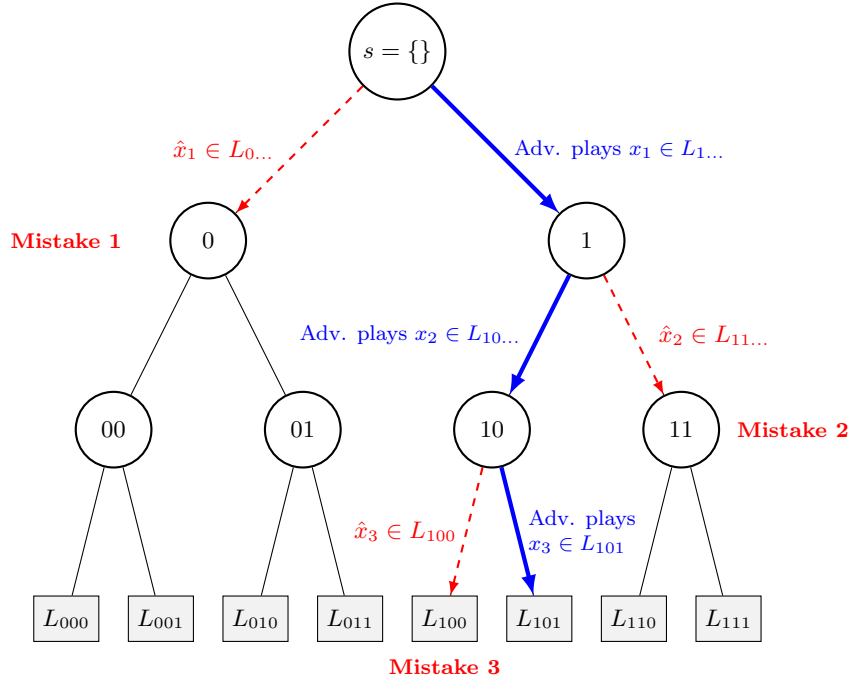
\begin{figure}[ht]
    \centering
    \begin{tikzpicture}[
        level distance=2.5cm,
        sibling distance=5cm,
        level 2/.style={sibling distance=2.5cm},
        level 3/.style={sibling distance=1.25cm},
        every node/.style={font=\small},
        consistent/.style={circle, draw=black, thick, minimum size=1cm, fill=white},
        leaf/.style={rectangle, draw=black, minimum size=0.6cm, fill=gray!10},
        erase underneath/.style={
            preaction={draw=white, line width=1.2pt, solid, shorten >=0.5pt, shorten <=0.5pt}
        },
        gen_guess/.style={->, red, dashed, thick, >=latex, erase underneath},
        adv_move/.style={->, blue, ultra thick, >=latex, erase underneath},
        mistake_label/.style={red, font=\footnotesize\bfseries, align=center},
        adv_label/.style={font=\footnotesize, align=left} 
    ]

    \node[consistent] (root) {$s=\{\}$}
        child { node[consistent] (n0) {$0$}
            child { node[consistent] (n00) {$00$}
                child { node[leaf] (L000) {$L_{000}$} }
                child { node[leaf] (L001) {$L_{001}$} }
            }
            child { node[consistent] (n01) {$01$}
                child { node[leaf] (L010) {$L_{010}$} }
                child { node[leaf] (L011) {$L_{011}$} }
            }
        }
        child { node[consistent] (n1) {$1$}
            child { node[consistent] (n10) {$10$}
                child { node[leaf] (L100) {$L_{100}$} }
                child { node[leaf] (L101) {$L_{101}$} }
            }
            child { node[consistent] (n11) {$11$}
                child { node[leaf] (L110) {$L_{110}$} }
                child { node[leaf] (L111) {$L_{111}$} }
            }
        };

    \draw[gen_guess] (root) -- node[left, xshift=-0.2cm] {$\hat{x}_1 \in L_{0\dots}$} (n0);
    \node[mistake_label, left=0.5cm of n0] {Mistake 1};
    
    \draw[adv_move] (root) -- node[adv_label, right, xshift=0.1cm] {Adv. plays $x_1 \in L_{1...}$} (n1);

    \draw[gen_guess] (n1) -- node[right, xshift=0.2cm] {$\hat{x}_2 \in L_{11\dots}$} (n11);
    \node[mistake_label, right=0.1cm of n11] {Mistake 2};
    
    \draw[adv_move] (n1) -- node[adv_label, left, xshift=-0.1cm] {Adv. plays $x_2 \in L_{10...}$} (n10);

    \draw[gen_guess] (n10) -- node[left] {$\hat{x}_3\in L_{100}$} (L100);
    \node[mistake_label, below=0.1cm of L100] {Mistake 3};
    \draw[adv_move] (n10) -- node[adv_label, right] {Adv. plays \\ $x_3 \in L_{101}$} (L101);

    \node[above=0.5cm of root, font=\bfseries] {Adversarial Strategy for Lemma \ref{lem:finite-class-lower-bound} ($n=8, m=3$)};

    \end{tikzpicture}
    \caption{Visualizing the proof of Lemma~\ref{lem:finite-class-lower-bound}. The tree represents the hierarchy of languages in $\mathcal{L}$. At each internal node $s$, the generator is forced to output a prediction consistent with exactly one branch, corresponding to halving the number of consistent languages (red dashed arrow). The adversary observes this and outputs an element that forces the true language to lie in the \emph{opposite} branch (blue solid arrow), guaranteeing a mistake. This process repeats for $m = \lfloor \log_2 |\mathcal{L}|\rfloor$ steps.}
    \label{fig:littlestone-tree}
\end{figure}

\begin{restatable}[Lower Bound]{lem}{finiteLB}\label{lem:finite-class-lower-bound}
    For any $N \in \mathbb{N}$, there exists a finite language class $\calL$ of size $N$ such that no generation algorithm can guarantee fewer than $\lfloor \log_2 N \rfloor$ mistakes on every adversarial stream consistent with an $L \in \calL$.
\end{restatable}

\begin{proof}[Proof of Lemma~\ref{lem:finite-class-lower-bound}]
Our goal is to show that for any $n > 0$, there exists a language class $\calL$ of size $n$ such that every generator will make at least $\log_2 n$ mistakes against an adversarial stream from a language $L \in \calL$. 

Let $m \in \Z_{\geq 0}$ be the largest $m$ such that $2^m \leq n$. Note that by definition, $m = \lfloor \log_2 n \rfloor$. 

We construct a collection of $2^m$ languages $\calL_m = \{L_v\}_{v \in \{0, 1\}^m}$ where each language is identified with a bit string $v \in \{0, 1\}^m$. If $2^m < n$, we additionally pad the language collection with dummy languages $L_{2^{m} + 1}, ..., L_{n}$ that are disjoint from all the languages in $\calL_m$ to ensure that the constructed collection has size exactly $n$. These will not be used in the construction of the hard adversarial stream. 

We now define the contents of $\calL_m$ as follows. Fix a subset of $\calX$, $B$, of size $\sum_{i = 1}^m 2^i = 2^{m + 1} - 2$. We label each point with a bitstring of length between 1 and $m$: $B := \bigcup_{i = 1}^m \{ x_p : p \in \{0,1\}^i \}$.

A point $x_p \in B$ lies in $L_v \in \calL_m$ if and only if $p$ is a prefix of $v$. 

In order to ensure each language has infinite support, we pad each $L_v$ with an infinite subset of $\calX$, disjoint from all other languages. 

We now describe how an adversary can construct an adversarial stream that forces at least $m$ mistakes by the generator. 

The adversary keeps track of a bitstring $s = \{\}$, which initially begins as empty. 

The adversary loops the following process for each time $t = 1, ....$:
\begin{enumerate}
    \item Observe the generator's output $\hat x_t$. 
    \item If $t = m + 1$, BREAK from the loop, and enumerate all remaining elements of $L_s$. 
    \item If the generator outputs $x_v \in B$ or $x \in L_v \setminus B$ such that $s$ is a prefix of $v$, observe the bit immediately proceeding $s$: $v[|s|]$ (0-indexed), and append the \emph{complement} of that bit to $s$. Otherwise, append an arbitrary bit to $s$. 
    \item Output $x_s \in B$. 
\end{enumerate}

It remains to show that the adversary's stream is valid, i.e. it consists of all unique elements and is consistent with a language in $\calL_m$, and that the generator makes $m$ mistakes. 

\paragraph{The adversary's stream is valid} First, we consider the final bitstring $s$ of length $m$ that the adversary constructs before hitting the break at time $m + 1$. At this point, $x_s \in B$ is consistent with exactly one language, $L_v$ where $v= s$, as it is the only string of length $m$ that contains $s$ as a prefix. Moreover, we are guaranteed that all $x_t$ for $t \leq m$ satisfied $x_t \in L_v$, because the bitstring $s$ at that time was a prefix of the final $s$ by construction, and all were unique, because $s$ was a different length at each step. By definition, all of the adversary's outputs after time $t = m$ are also unique and consistent. Therefore, we conclude that the adversarial stream is valid, and moreover is actually a valid \emph{enumeration} of $L_v$. 

\paragraph{The generator makes $m$ mistakes. } We can show that every $\hat x_t$ for $t \leq m$ is a mistake. In particular, at the beginning of timestep $t$, the set of potential languages that are consistent with the adversarial stream so far are exactly all of the languages $L_v$ where $s$ is a prefix of $v$. By construction, there are $2^{m - t + 1}$ such languages. The adversary's $x_t$ is defined to be $x_{s||0}$ or $x_{s||1}$, where seeing either of these points reduces the set of consistent languages by half to $2^{m - t}$, and moreover these two consistent sets are totally disjoint. 

If in the check in step 3, the generator outputs $x_v \in B$ or $x \in L_v \setminus B$ such that $s$ is a prefix of $v$, this means that its output element can only be consistent with languages in one of these two disjoint sets. By playing $x_{s||\neg v[|s|]}$, the adversary reveals the true language to be in the other set of consistent languages, and thus the generator has made a mistake. If the generator's point does not satisfy the first condition, it means the generated element is not consistent with any of the $2^{m - t - 1}$ languages still consistent with the adversary's stream, and thus is also automatically a mistake regardless of what the adversary outputs. 

Thus, we conclude that in every step from $t = 1$ to $t = m$, the generator's move does not lie in the target language $L_v$, and thus the generator makes $m = \lfloor \log_2 n\rfloor$ mistakes.     
\end{proof}

\section{Mistake-Bounded Non-Uniform Generation}\label{sec:infinite-classes}
We now turn to the general setting of countably infinite language classes $\calL = \{L_1, L_2, \dots\}$. Here, we cannot hope for a uniform bound on mistakes across the entire class. Instead, we seek a \emph{non-uniform} guarantee that depends on the index $i$ of the target language $L_i$. Existing work by \cite{li2024generation} and \cite{charikar2024exploring} focuses on minimizing the \emph{time of last mistake}. For example, the Greedy algorithm of \cite{charikar2024exploring} guarantees convergence in terms of the \emph{non-uniform complexity of $L_i$}, $m(L_i)$, a structural parameter representing the complexity of $L_i$'s intersections with preceding languages. 

\begin{definition}[Non-Uniform Complexity]
    Given an infinite stream of languages $\calL = L_1, L_2, ...$, the \emph{non-uniform complexity} of $L_i$ is defined as size of the largest finite subcollection of $(L_1, ..., L_i)$ that contains $L_i$:
    \[m(L_i) := \max_{\calL' \subseteq \{L_1, ..., L_{i - 1}\}}\{|\bigcap_{L \in \calL'} L \cap L_i| : |\bigcap_{L \in \calL'} L \cap L_i| < \infty\}.\]
\end{definition}

In this section, we present two contrasting results. First, we show that if one prioritizes minimizing mistakes, we can achieve a mistake bound of $O(\log_2 i)$, i.e. with only logarithmic dependence on the index. However, we then prove that this comes at a cost: no algorithm can simultaneously achieve this logarithmic mistake bound and the convergence time guaranteed by \cite{charikar2024exploring}. 

\subsection{Minimizing Mistakes without Convergence}
To minimize the total number of mistakes, we revisit our General Recipe (Algorithm~\ref{alg:non-uniform-mistake-bound}). In the finite case, we used a uniform prior. For infinite streams, we must adopt a prior that decays sufficiently fast to ensure finite capacity, yet slowly enough to not penalize later languages too heavily. We instantiate Algorithm~\ref{alg:non-uniform-mistake-bound} with polynomially decaying weights $w_0(i) = \frac{1}{i^2}$ (noting that $\sum \frac{1}{i^2} < \infty$). and a growth function that doubles the number of languages considered at each step, $f(t) = 2^t$. 

Under this instantiation, our algorithm yields the following guarantee:

\begin{restatable}{thm}{NUlogmistakes}\label{thm:non-uniform-mistake-bound} 
    Let $\calL = \{L_1, L_2, ...\}$ be a countably infinite class of languages under a fixed enumeration. For any language $L_i \in \calL$, the instantiation of Algorithm~\ref{alg:non-uniform-mistake-bound} described above guarantees a non-uniform mistake bound of: \[ M_{\mathrm{non\text{-}unif}}(\calG, i) \leq 3\log_2 i + \log_2(\pi^2/6) = O(\log_2 i). \] 
\end{restatable}

The proof follows immediately from an appropriate instantiation of Theorem~\ref{thm:general-recipe}. 

\begin{proof}[Proof of Theorem~\ref{thm:non-uniform-mistake-bound}]
    To prove the theorem, we instantiate Theorem~\ref{thm:general-recipe} with the weight function $w_0(i) = 1/i^2$ and growth function $f(t) = 2^t$. By standard results due to \cite{euler1740summis}, we have that 
    $\sum_{i = 1}^\infty 1/i^2 = \pi^2/6$, and thus can set $W = \pi^2/6$. 

    Under this definition of $f$, $f^{-1}(i)$, i.e. the largest $t$ such that $f(t) < i$, is upper bounded by $\log_2(i)$. 

    Thus, for this particular instantiation of weight functions and growth function, we can guarantee that Algorithm~\ref{alg:non-uniform-mistake-bound} guarantees a non-uniform mistake bound of
    \begin{align*}
        m(\mathcal{G}, L_i) &\leq f^{-1}(i) + \lfloor \log_2 W/w_0(i) \rfloor \\
        &\leq \lfloor \log_2 i \rfloor +  \log_2 \pi^2i^2/6 \\
        &= \log_2 \pi^2/6 + 3\log_2 i\\
        &= O(\log_2 i).
    \end{align*}
\end{proof}

\subsection{Simultaneously Bounding Mistakes and Convergence}
While the strategy above guarantees low error, it does not prioritize \emph{when} those errors occur. In particular, by prioritizing bounding mistakes on languages late in the stream, it may be slow to converge on early languages.

We next present an algorithm that can non-uniformly generate from $\calL$ while simultaneously bounding the total number of mistakes. However, in order to balance these two goals, we must weaken our mistake bound to only guaranteeing a mistake bound that is linear in the index of the target language. To obtain this guarantee we analyze a modification of the greedy non-uniform generation algorithm by \cite{charikar2024exploring}, termed \textsc{Modified-Greedy}. This algorithm prioritizes maintaining consistency with the earliest consistent languages. 

\begin{restatable}{lem}{CMlinearmistakes}\label{lem:cm-mistake} 
    There exists a non-uniform generation algorithm (\textsc{Modified-Greedy}) with a time-of-last-mistake guarantee of $\max\{i-1, m(L_i) + 1\}$ and a non-uniform mistake bound of: \[ M_{\mathrm{non\text{-}unif}}(\calG, i) \leq \min\{2(i-1), \max\{i-1, m(L_i)+1\}\}. \] 
\end{restatable}

Note that the mistake bound here is linear in $i$ ($2(i-1)$), which is exponentially worse than the $O(\log i)$ bound in Theorem~\ref{thm:non-uniform-mistake-bound}. 

We refer to the algorithm we will use to obtain our guarantees as \textsc{Modified-Greedy} (Algorithm~\ref{alg:greedy-non-uniform-gen-alg}).  The algorithm as we describe it differs from the original definition provided by \cite{charikar2024exploring}, with the change highlighted in red. However, while the modification means that the actions taken by the two algorithms will be different, the modified algorithm still satisfies the last-mistake-bound promised by \cite{charikar2024exploring}. For completeness, we re-prove the last-mistake-time bound for the modified algorithm in Lemma~\ref{lem:cm-last-mistake-time}. 

\begin{algorithm}[H] 
    \DontPrintSemicolon
    \SetAlgoNoLine
    
    \caption{\textsc{Modified-Greedy} (c.f. Algorithm in Theorem 6 of \cite{charikar2024exploring})}
    \label{alg:greedy-non-uniform-gen-alg}
    
    \KwIn{Universe $\calX$, infinite stream $\mathcal{L} = L_1, L_2, ...$}
    
    \For{$t = 1, \dots, \infty$}{
        Initialize $I_t = \color{red}{\calX \setminus x_{1:t-1}}$ .\;\\
        \For{$i = 1, ..., t$}{
            \If{$x_{1:t-1} \subseteq L_i$ and \color{red}{$|L_i \cap I_t| > 0$} }{$I_t \leftarrow I_t \cap L_i$}
        }
        Output $\hat{x}_t \in I_t$.
    }
\end{algorithm}

\begin{restatable}{lem}{cmLastMistake}\label{lem:cm-last-mistake-time}
    For any infinite stream $\calL = L_1, L_2, ...$, \textsc{Modified-Greedy}(Algorithm~\ref{alg:greedy-non-uniform-gen-alg}) guarantees a last mistake time $t(L_i) \leq \max\{i - 1, m(L_i) + 1\}$ for any adversarial stream consistent with $L_i \in \calL$. 
\end{restatable}

\begin{proof}
    Consider any $t \geq \max\{i, m(L_i) + 2\}$. This means that during step $t$, \textsc{Modified-Greedy} considers $L_i$ as one of the languages to possibly intersect with $I_t$. We are also guaranteed that $x_{1:t-1} \subseteq L_i$ because the adversary is consistent with $L_i$.

    Consider the point at which $L_i$ is considered by the algorithm during step $t$. By definition, there is some $J \subseteq [i - 1]$ such that $x_{1:t-1} \subseteq L_j$ for each $j \in J$, and $I_t = \bigcap_{j \in J}L_j \setminus x_{1:t- 1}$. Thus, $x_{1:t-1} \subseteq \bigcap_{j \in J} L_j \cap L_i$. 

    If $|I_t \cap L_i| > 0$, the algorithm is guaranteed to play a $\hat{x}_t \in L_i$, and thus does not make a mistake. We claim that it cannot be the case that $|I_t \cap L_i| = 0$. This is because this would imply that 
    \[t-1 \leq \bigcap_{j \in J} L_j \cap L_i \leq m(L_i),\]
    but this is a contradiction, because we assumed that $t \geq m(L_i) + 2$.

    Thus, we conclude that \textsc{Modified-Greedy} will not make a mistake for all \[t \geq \max\{i, m(L_i) + 2\},\] and thus the last mistake time is at most $\max\{i - 1, m(L_i) + 1\}$. 
\end{proof}

We are now ready to prove Lemma~\ref{lem:cm-mistake}, which further bounds the number of mistakes made by \textsc{Modified-Greedy} to show that it satisfies both a time-of-last-mistake bound and a total mistake bound. 

\begin{proof}[Proof of Lemma~\ref{lem:cm-mistake}]
    Our proof proceeds by showing that every mistake incurred by \textsc{Modified-Greedy} on a language $L_i$ is accompanied by eliminating some language in the prefix $L_1, ..., L_{i - 1}$. Because there are only $i - 1$ such languages, the number of mistakes on $L_i$ is bounded by $2(i - 1)$, where the double factor comes because we don't consider $L_i$ until step $i$, and so could make mistakes during the $i - 1$ steps before that.  
    
    At any step $t \geq i$, we have two cases. Either $\hat x_t \in L_i$, in which we can't make a mistake, or $\hat {x}_t \not \in L_i$, in which case the algorithm makes a mistake when the true language is $L_i$. 

    if $\hat x_t \not\in L_i$, this means that at the point in the algorithm during step $t$ when $I_t \cap L_i$ was considered, we had $|I_t \cap L_i| = 0$. Note that by definition, $I_t = \bigcap_{j \in J} L_j \setminus x_{1:t-1}$ for some $J \subseteq [i - 1]$ satisfying $x_{1:t - 1} \subseteq L_j$ for each $j \in J$. 

    Thus, because $x_t \in L_i$, we have $x_t \not\in I_t$, which implies that there exists some $j < i$ such that $x_{1:t-1} \subseteq L_j$, but $x_t \not\in L_j$. This implies that $L_j$ becomes inconsistent at step $t$, and is no longer considered after $t$. 

    Because there are only $i - 1$ languages before $i$, this means that the algorithm can make at most $i - 1$ mistakes after starting to consider $L_i$, and thus can make at most $2(i - 1)$ mistakes. 

    Finally, the last-mistake-time of $\max\{i - 1, m(L_i) + 1\}$ follows from Lemma~\ref{lem:cm-last-mistake-time}, and because the total number of mistakes is upper-bounded by the last-mistake time, we can refine the total number of mistakes to $\min\{ 2(i - 1), \max\{i - 1, m(L_i) + 1\}\}$. 
\end{proof}

\subsection{Tradeoff Between Last-Mistake-Time and Total Mistakes}
This raises a natural question: is there a generation strategy that achieves our optimal mistake bound of $O(\log_2 i)$ while also preserving \cite{charikar2024exploring}'s last-mistake-time guarantee? 

We answer in the negative. We prove a fundamental barrier: to achieve the optimal convergence time, an algorithm must be willing to make mistakes on later languages in order to prioritize converging on languages appearing earlier in the stream. Conversely, to achieve logarithmic mistakes, an algorithm must start to prioritize later languages over early languages when they start to accumulate many mistakes, potentially delaying convergence time. 

\begin{restatable}{thm}{NUtradeoff}\label{thm:non-uniform-tradeoff}
    There exists an infinite language class equipped with a fixed enumeration $\calL = \{L_1, L_2, ...\}$ such that any algorithm that guarantees a last-mistake time of $t(L_i) \leq O(\max\{i, m(L_i)\})$ on all $i \in\N$ must suffer a number of mistakes that grows linearly with the language index. Specifically, for every $i \geq 2$, 
    \[M_{\mathrm{non\text{-}unif}}(\calG, i) \geq i - 1.\] 
\end{restatable}

The key intuition for the proof is as follows: we construct a sequence of languages sharing nested, rapidly growing finite prefixes. By sequentially enumerating these shared elements, the adversary forces the generator into a recurring dilemma at each prefix boundary. At the $k$th boundary, the unseen points available to the generator are either points solely in the language $L_{k}$, or a point shared by all proceeding languages $L_j$ for $j \geq k$, but outside of $L_{k}$. If the generator predicts a point outside of $L_{k}$, the adversary halts and declares that language the target. This forces a late mistake that violates the time-of-last-mistake bound for $L_{k}$. On the other hand, if the generator safely plays from $L_k$ at each boundary, it means that by the time it gets to the $i$th boundary, it has incurred a mistake on $L_i$ at every boundary leading up to $i$, resulting in a linear number mistakes if the adversary halts and declares $L_i$ the target. 

\begin{proof}[Proof of Theorem~\ref{thm:non-uniform-tradeoff}]
    We construct our hard language stream on the countably infinite $\calX = \N \times \N$, and thus refer to elements $x \in \calX$ as pairs $(i, j) \in \N \times \N$. 

    Fix an arbitrarily large $n \in \mathbb{N}$. 
    
    We define language $L_i$ as follows for any $i \geq 1$:

    \[L_i = \{(i, j) : j \in \N\} \cup \bigcup_{k \in [i - 1]}\{(k, j) : j \in [n^k]\} \]

    Note that this is a valid language class, because each language has infinite support, and for any particular language $L_i$, 
    \[m(L_i) = |\bigcup_{k \in [i - 1]}\{(k, j) : j \in [n^k]\}| = \sum_{j = 1}^{i - 1}n^j\]

    With $m(L_1) = 0$. Similarly note that 

    \[|\bigcap_{j \geq i} L_j| = m(L_i) + n^i = \omega(m(L_i)).\]

    Given this language class, and any index $i^* \geq 2$, we describe an adversary strategy that either results in a time-of-last-mistake equal to $m(L_i) + n^i = \omega(m(L_i))$ for some language $i$, or forces the number of mistakes on $i^*$ to be at least $i^* - 2$. 
    
    We have the adversary first enumerate all $(1, j)$ for each $j \in [n]$, then $(2, j)$ for $j \in [n^2]$, $(3, j)$ for each $j \in [n^3]$,  etc.

    Let $t_i$ be the timestep at which the adversary would output $x_t = (i, 1)$ for each $i \in \mathbb{N}$. Note that before seeing $(i, 1)$, all of the adversary's outputs thus far are consistent with $L_{j}$ for all $j \geq i - 1$. At each such timestep, starting at $t_2$, the adversary observes $\hat{x}_{t_i}$, and proceeds as follows:
    \begin{enumerate}
        \item if $\hat{x}_{t_i} \not\in L_{i-1}$, the adversary stops its enumeration, and instead enumerates all of the remaining elements of $L_{i-1}$. 
        \item Otherwise, if $i = i^*$, the adversary stops and enumerates all remaining elements of $L_{i^*}$. 
    \end{enumerate}

    Note that we are guaranteed to trigger one of steps 1 or 2 at some point during the game, as $i^*$ is finite. If Step 1 executes, this implies a contradiction to the last-mistake-time bound for language $L_{i - 1}$, as all of the adversary's outputs thus far and in the future are consistent with $L_{i - 1}$, and the generator makes a mistake on $L_{i-1}$ at time
    \[t_i = 1 + \sum_{j = 1}^{i -1} n^j = m(L_{i -1}) + n^{i-1} + 1 = \omega(m(L_{i - 1})).\]

    Note that by construction, at timestep $t_i$, we have that for any $j \geq i$, $L_j \cap L_{i-1}\setminus x_{1:t_i - 1} = \emptyset$, and thus if the generator plays $\hat x_{t_i} \in L_{i - 1}$, we must have $\hat x_{t_i} \not\in L_j$ for all $j \geq i$, and thus this counts as a mistake if any such $j$ ends up being the target language. 

    If step 1 never executes, this means that the generator plays from $L_{i - 1}$ at each $t_i$ for $i = 2, ...., i^*$, and thus makes at least $i^* - 1$ mistakes on $L_{i^*}$.

    Thus, in either case the generator is forced to either contradict the last-mistake-time-bound of $m(L_i)$ for some $i \leq i^*$, or make a linear number of mistakes on $i^*$. 
\end{proof}

\section{Mistake-Bounded Noisy Generation}\label{sec:noisy-extension}
The algorithms presented in Sections~\ref{sec:finite-classes} and \ref{sec:infinite-classes} rely on eliminating candidate languages as soon as they become inconsistent with the adversary's outputs. While this approach works well in a setting where we assume the adversary is perfectly consistent with a target language, it's also natural to explore what guarantees are possible against \emph{noisy} adversaries, who may not be perfectly consistent with the target language. 

This imperfect setting has been studied in the context of last-mistake-time guarantees by a number of works. \cite{raman2025generation} first proposed a noisy setting where the adversary may insert noise at an unknown but finite number of steps. More recently, \cite{bai2025language} further explore the limits of generating in the presence of a finite amount of noise, while \cite{mehrotra2025language} show a surprising result that any countable class can be generated in the limit in the presence of potentially infinite but vanishing noise. 

In this section, we show that noise-dependent mistake bounds for the noisy setting can be derived by leveraging results from the learning from demonstrations setting. The key idea behind these results is to relax the weight updates to downweight languages when the adversary makes a mistake, rather than setting the weight to 0 and removing them from consideration. 

\begin{restatable}[Noisy Mistake Bound for Finite Classes]{lem}{finitenoise}\label{lem:noisy-finite}
    Given a finite language class $\calL$ and $\gamma \in (0, 3/4]$, given an adversarial stream $x_1, x_2, ...$, not necessarily consistent with any $L \in \calL$, denote the number of mistakes the stream has made on $L$ at time $t$ as
    \[M_L(x_{1:t}) = \sum_{i = 1}^t \mathbf{1}[x_i \not\in L].\]
    Similarly, denote the learner's mistakes up to time $t$ as $M_L(\hat x_{1:t})$. Then, there exists an algorithm that guarantees for any $L$, 
    \[M_L(\hat x_{1:t}) \leq (1 + 2\gamma)M_L(x_{1:t}) + \log_2 |\calL|/\gamma.\]
\end{restatable}

This noisy guarantee follows from the results of \cite{joshi2025learning} on learning from demonstrations. In Appendix~\ref{sec:lfd-noisy}, we show its derivation. If we instantiate this result with $\gamma = 1/2$ and assume our adversary is guaranteed to only have a finite amount of noise, then this result guarantees that our generator will only make a finite number of mistakes, no matter the amount of noise. In particular, its mistakes are bounded by $2(M + \log_2|\calL|)$, where $M$ is the amount of noise inserted by the adversary.

While the results of \cite{joshi2025learning} do not extend to infinite streams of reward functions, we provide an extension of their result (Theorem~\ref{thm:lfd-inf-stream-noisy}) that allows us to derive non-uniform mistake bounds for infinite streams of languages. Proper instantiation of the theorem's parameters gives the following guarantee for language generation in the presence of noisy adversaries (see Appendix~\ref{sec:lfd-noisy} for details):

\begin{restatable}[Noisy Mistake Bound for Infinite Streams]{lem}{noisyInf}\label{lem:noisy-infinite}
    Suppose we are given an infinite\\ stream of languages $\calL= L_1, L_2, ...$ and $\gamma \in (0, 3/4]$. Then, there exists a generation algorithm such that for every adversarial stream $x_1, x_2, ...$, not necessarily consistent with any $L_i \in \calL$ and $i, T \in \N$, 
    \[\sum_{t = 1}^T \mathbf{1}[\hat x_t \not\in L_i] \leq (1 + 2\gamma)\sum_{t = \lfloor \log_2 i\rfloor}^T \mathbf{1}[x_t \not\in L_i] + (1 + 2/\gamma)\log_2 i.\]
\end{restatable}

As in the finite case, we note that this result implies that when the adversary's stream only has a finite amount of noise, the generator is also guaranteed to make at most a finite number of mistakes.

\subsection{Comparison to existing time-of-last-mistake bounds}

The guarantee of Lemma~\ref{lem:noisy-finite} echoes that of~\cite{raman2025generation}, who show that whenever the adversary only adds a finite amount of noise, the time-of-last-mistake can be bounded for any finite language class. Similarly, Lemma~\ref{lem:noisy-infinite} guarantees in cases where the adversary makes a finite amount of mistakes, we can guarantee a non-uniform but finite number of mistakes for any countable class. This can be compared to the results of \cite{mehrotra2025language}, who show that any countable language can be generated in the limit in the presence of a finite amount of noise.  We leave as a key open question whether the time-of-last-mistake bounds given by \cite{raman2025generation} and \cite{mehrotra2025language} can be achieved simultaneously with the mistake bounds guaranteed in Lemmas~\ref{lem:noisy-finite} and \ref{lem:noisy-infinite}, respectively.

Interestingly, while \cite{mehrotra2025language} are able to guarantee a finite number of mistakes (upper-bounded by the finite last-mistake-time) in the presence of infinite but vanishing noise, our mistake bounds do not necessarily guarantee a finite number of mistakes in such a case, just that the rate of mistakes made by the learner is no worse than that of the adversary. We leave developing improved algorithms with a bounded number of mistakes in the $o(1)$-noise case, along with further analysis of the interaction between time-of-last-mistake and number of mistakes, as an open question for future work.

\section{Open Directions}
Our introduction of mistake bounds as an alternative metric of success in language generation raises several key questions. We note three of particular interest. (1) Our bounds for finite classes rely on cardinality ($\log_2 |\calL|$). In online classification, the \emph{Littlestone dimension} characterizes learnability independent of class size \citep{littlestone1988learning}. Does there exist an analogous combinatorial dimension $d(\calL)$ for language generation that characterizes the optimal mistake bound more tightly than cardinality or closure dimension? (2) Theorem~\ref{thm:non-uniform-tradeoff} identifies a conflict between minimizing mistakes ($O(\log i)$) and minimizing time-of-last-mistake ($O(m(L_i))$). What is the full Pareto frontier of simultaneous mistake and convergence guarantees? (3) Finally, in the noisy setting, it remains an open problem whether our mistake bounds can be achieved \emph{simultaneously} with bounded time-of-last-mistake guarantees.

\section*{Acknowledgments}
JK is supported in part by the Simons Foundation Collaboration on the Theory of Algorithmic Fairness and a grant from the MacArthur Foundation. CP is supported by the Simons Foundation Collaboration on the Theory of Algorithmic Fairness and the Apple Scholars in AI/ML PhD fellowship. OR is supported by the Simons Foundation Collaboration on the Theory of Algorithmic Fairness, and the Simons Foundation investigators award 17351.

\bibliographystyle{plainnat}
\bibliography{ref}

\medskip

\appendix
\section{Proofs from Section~\ref{sec:general-recipe}}\label{sec:general-recipe-pf}

In this section, we prove Theorem~\ref{thm:general-recipe}, which shows that our general algorithm (Algorithm~\ref{alg:non-uniform-mistake-bound}) guarantees a mistake bound that depends on the choice of initial weights $w_0$ and growth function $f$. We restate the theorem here for readability:

\generalrecipe*

Our proof relies on the following lemma, which establishes a key property about how the weights in Algorithm~\ref{alg:non-uniform-mistake-bound} evolve:

\begin{restatable}{lem}{weightfnBd}\label{lem:weight_fn_bound}
    Define $W_t := \sum_{i = 1}^{f(t+1)}w_t(i)$. Then, for any step $t \geq 1$, 
    \[W_t \leq W_{t - 1} + \sum_{i = f(t) + 1}^{f(t+1)} w_0(i).\]
\end{restatable}

\begin{proof}
    We re-express $W_t$ by decomposing the weights:
    \begin{align*}
        W_t &= \sum_{i = 1}^{f(t)} w_t(i) + \sum_{i = f(t) + 1}^{f(t + 1)} w_t(i) \\
        &\leq \sum_{i = 1}^{f(t)} w_t(i) + \sum_{i = f(t) + 1}^{f(t + 1)} w_0(i) \tag{Definition of weight initialization}
    \end{align*}

    Thus, it suffices to prove that $\sum_{i = 1}^{f(t)} w_t(i) \leq W_{t - 1} = \sum_{i = 1}^{f(t)} w_{t-1}(i) $. 

    Let $I_t = \{1 \leq i \leq f(t) : w_t(i) \neq 0\}$. Using this set, we re-express our sum as 

    \begin{align*}
        \sum_{i = 1}^{f(t)} w_t(i) &= \sum_{i \in I_t} w_t(i)\\
        &= 2\sum_{i \in I_t, \hat{x}_t \not\in L_i}w_{t-1}(i) + \sum_{i \in I_t, \hat{x}_t \in L_i}w_{t-1}(i) \tag{Definition of weight update}\\
        &= \sum_{i \in I_t}w_{t-1}(i) + \sum_{i \in I_t, \hat{x}_t \not\in L_i}w_{t-1}(i)
    \end{align*}

    Let $I_{t - 1} = \{i \leq f(t) : w_{t-1}(i) \neq 0\}$. Note that $I_{t} \subseteq I_{t - 1}$. Because all weights are non-negative, we get an upper bound by replacing $I_t$ with $I_{t -1}$ in the above expression:

    \begin{align*}
        &\sum_{i \in I_t}w_{t-1}(i) + \sum_{i \in I_t, \hat{x}_t \not\in L_i}w_{t-1}(i) \\
        &\leq \sum_{i \in I_t}w_{t-1}(i) + \sum_{i \in I_{t-1}, \hat{x}_t \not\in L_i}w_{t-1}(i)\\
        &= \sum_{i \in I_{t-1}, x_t \in L_i}w_{t-1}(i) + \sum_{i \in I_{t-1}, \hat x_t \not\in L_i}w_{t-1}(i)\\
        &= W_{t -1} + \sum_{i \in I_{t-1}, x_t \in L_i}w_{t-1}(i) - \sum_{i \in I_{t-1}, \hat x_t \in L_i}w_{t-1}(i)
    \end{align*}

    By the definition of Algorithm~\ref{alg:non-uniform-mistake-bound}, $\hat x_{t}$ is picked according to the rule
    \[\hat x_{t} = {\arg\max}_{x \in \mathcal{X} \setminus x_{1:t-1}}\sum_{i = 1}^{f(t)}w_{t - 1}(i)\mathbf{1}[x \in L_i] = {\arg\max}_{x \in \mathcal{X} \setminus x_{1:t-1}}\sum_{i \in I_{t - 1}, x \in L_i}w_{t - 1}(i).\]

    And by assumption, the adversary's output $x_t$ comes from the set $\calX \setminus x_{1:t-1}$. Thus, we are guaranteed that $\sum_{i \in I_{t-1}, x_t \in L_i}w_{t-1}(i) - \sum_{i \in I_{t-1}, \hat x_t \in L_i}w_{t-1}(i) \leq 0$ and can conclude that $\sum_{i = 1}^{f(t)} w_t(i)$ is upper-bounded by $W_{t - 1}$, completing the proof. 
\end{proof}

    We now apply this lemma to prove Theorem~\ref{thm:general-recipe}.

    \begin{proof}[Proof of Theorem~\ref{thm:general-recipe}]
        We first note that unrolling Lemma~\ref{lem:weight_fn_bound} gives a constant upper bound on any $W_t$:
        \[W_t \leq W_{t - 1} + \sum_{i = f(t) +1}^{f(t+ 1)} w_0(i) \leq \sum_{i = 1} \leq \sum_{i = 1}^{f(t + 1)} w_0(i) \leq W.\]

        For a particular true language $L_i$, the first step where $L_i$ is considered in the argmax by the algorithm is at $f^{-1}(i) + 1$. Up until this point, the algorithm could have made at most $f^{-1}(i)$ mistakes. 

        We now focus on how many additional mistakes the algorithm can make after first introducing $L_i$. Suppose that at a timestep $T$, we have that
        \[\sum_{t = f^{-1}(i) + 1}^T \mathbf{1}[\hat x_t \not\in L_i = M,\]
        i.e. the algorithm has made $M$ mistakes after first introducing $L_i$. By definition of our update rule, for each of these timesteps, we had $x_t \in L_i$ because it is the true language, and $\hat x_t \not\in L_i$. Thus, the initial weight of $L_i$ has doubled $M$ times, and we have
        \[w_T(i) = 2^Mw_0(i) \leq W_T \leq W.\]
        Where we get an upper bound of $W$ from Lemma~\ref{lem:weight_fn_bound}. Solving for $M$, we conclude that
        \[M \leq \lfloor \log_2 W/w_0(i)\rfloor.\]
        Because this holds for any $T$, we conclude that the algorithm makes at most $\lfloor \log_2 W/w_0(i)\rfloor$ additional mistakes after the introduction of $L_i$, and thus makes at most $f^{-1}(i) + \lfloor \log_2 W/w_0(i)\rfloor$ mistakes in total.
    \end{proof}

\section{Mistake Bounds via Learning from Demonstrations}\label{sec:lfd-connections}

Recent work by \cite{joshi2025learning} studies the setting of Learning from Demonstrations (LfD). In this section, we discuss how the problem of mistake bounded language generation can be reduced to this setting, allowing us to use the techniques of~\cite{joshi2025learning} to derive mistake bounds. 

The online setting of learning from correct demonstrations can be thought of as a contextual bandit problem. Let $\calX$ be a set of contexts, and $\calY$ be a set of actions. A reward function $r: \calX \times \calY \rightarrow \{0, 1\}$ specifies a binary reward for each context/action pair.  

The online interaction between the learner and adversary proceeds as follows for each round $t= 1, 2, 3, ..., T$:
\begin{enumerate}
    \item The learner receives context $x_t \in \calX$ from the adversary.
    \item The learner outputs an action $\hat y_t \in \calY$.
    \item The learner receives a demonstration $y_t \in \calY$ from the adversary. \emph{The learner receives no other feedback.}
\end{enumerate}

In the setting of learning from an optimal demonstrator, or what can be thought of as the ``realizable'' setting of the problem, the learner has access to a collection of reward functions $\calR$. The adversary is assumed to be optimal, in that there exists some true reward function $r^* \in \calR$ such that for all timesteps $t$, $r^*(x_t, y_t) = 1$. We assume that all reward functions and contexts have at least one action that results in non-zero reward. 

Crucially, this true reward $r^*$ is unknown to the learner and the learner never receives feedback about the reward values of its output actions. Nevertheless, the learner's performance is evaluated with respect to this true reward. We say that the learner makes a \emph{mistake} at timestep $t$ if $r^*(x_t, \hat y_t) = 0$. 

Typical online learning strategies typically rely on feedback at every timestep in order to update a policy toward more successful outputs. The lack of such feedback in this online demonstration setup thus makes it seem difficult, and surprising that a learner could hope to bound its mistakes. However, this is the main result of \cite{joshi2025learning}, who show that for finite reward classes $\calR$, there exists an algorithm that guarantees a learner makes no more than $\log_2 |\calR|$ mistakes on any $\calR$:

\begin{restatable}[\cite{joshi2025learning}, Theorem 4]{thm}{joshilfdguarantee}
    There exists an online algorithm such that given any online sequence $((x_t, y_t))_{t \in \mathbb{N}}$ such that there exists $r^* \in \calR$ with $r^*(x_t, y_t) = 1$ for all $t$, the algorithm makes at most $\log_2|\calR|$ mistakes, e.g. 
    \[\sum_{t \in \mathbb{N}} (1 - r^*(x_t, \hat y_t)) \leq \log_2 |\calR|.\]
\end{restatable}

\subsection{Connecting Mistake-Bounded Language Generation to Learning from Demonstrations}

In this section, we describe how guarantees for learning from demonstrations can be translated to mistake bounds for language generation. At first glance, one might assume that language generation is a special case of learning from demonstrations where there is only a single context, and each reward function is an indicator of a particular language. 

However, language generation's requirement that the generator output \emph{unseen} elements means that this is not quite the case, as a learner in the single context setting could succeed by continuously outputting the adversary's initial output. To draw a true comparison, we must carefully construct contexts and reward functions based on a language class to ensure a learner can only receive high reward by outputting new elements. We describe the full reduction in the theorem below. 

We present the reduction in the most general terms possible, i.e. for potentially non-uniform mistake bounds dependent on the position of the reward function in the stream and an imperfect adversary, as well as in terms of the timestep $T$. For the remainder of this section, we denote the universe of elements used in language generation as $\mathcal{U}$, with elements $u \in \mathcal{U}$, to differentiate from the context space in LfD.

\begin{restatable}{lem}{reduction}\label{lem:reduction}
    Suppose we have access to a learning from demonstrations algorithm such that for any choice of contexts $\calX$, label set $\calY$, and reward class $\calR = r_1, r_2, ...$, presented as a (potentially infinite) stream of functions, the algorithm guarantees that for any adversarial stream $(x_1, y_1), (x_2, y_2), ...$, the following mistake bound is satisfied for each $r_i \in \calR$ and $T \in \mathbb{N}$:
    \[\sum_{t = 1}^T (1 - r_i(x_t, \hat y_t)) \leq M(i, T, \sum_{t = 1}^T (1 - r_i(x_t, y_t))).\]

    Then, for any language class $\calL = L_1, L_2, ...$, we can use this algorithm to construct a generator whose outputs are guaranteed to satisfy for any $L_i \in \calL$ and $T \in \mathbb{N}$, 

    \[\sum_{t = 1}^T \mathbf{1}[\hat u_t \not\in L_i \setminus u_{1:t-1}] \leq M(i, T, \sum_{t = 1}^T \mathbf{1}[u_t \not\in L_i \setminus u_{1:t-1}]).\]
\end{restatable}

\cite{joshi2025learning}'s result for finite classes and perfect demonstrators can be interpreted in this setting as guaranteeing $M(i, T, 0) \leq \log_2 |\calR|$ for all $i \in [|\calR|]$ and $T \in \N$, and thus bounds the number of mistakes in the language setting by $M(i, T, 0) \leq \log_2 |\calL|$ as well. 

In the setting of sub-optimal demonstrators, where the demonstrator may not always give a demonstration with high reward, \cite{joshi2025learning} show that when the reward class is finite, the sub-optimality of the learner's reward with respect to $r_i$ can be upper-bounded in terms of the class size and adversary's suboptimality, thus guaranteeing 
\[M(i, T, \sum_{t = 1}^T (1 - r_i(x_t, y_t))) \leq (1 + 2\gamma)\sum_{t = 1}^T (1 - r_i(x_t, y_t)) + \log_2|\calR|/\gamma \]
for any choice of $\gamma \in (0, 1)$. This is the guarantee we leverage to provide mistake-bound guarantees for noisy adversaries in Section~\ref{sec:noisy-extension}. 

The results of \cite{joshi2025learning} only apply to settings with finite reward classes, where we can hope for a uniform guarantee over all possible reward functions. However, the setting of non-uniform language generation requires mistake bounds that can be guaranteed non-uniformly over a countably infinite stream of languages. 

While our reduction still applies to this case, there are no existing results that provide LfD guarantees for an infinite stream of rewards. In Section~\ref{sec:lfd-inf-rewards}, we give a new algorithm for the LfD setting that can give guarantees for an infinite stream of reward functions, allowing us to derive mistake bounds in the non-uniform language generation setting. 

We now present the proof of Lemma~\ref{lem:reduction}. 

\begin{proof}[Proof of Lemma~\ref{lem:reduction}]
    Given a language class $\calL = L_1, L_2, ...$ for a universe $\mathcal{U}$ and a general algoithm for LfD, we describe how to leverage the LfD algorithm to generate with bounded mistakes. 

    We first define the context space $\calX$ and label space $\calY$ for the algorithm. We let $\calY = \mathcal{U}$, and $\calX = \{x_{u_{1:n}} : n \geq 0, u_1, ..., u_n \in \mathcal{U}\}$, e.g. a unique context for each finite string of elements from $\mathcal{U}$. We refer to elements in $\calX$ as $x_s$, where $s$ is some finite string of elements $\{u_1, ..., u_n\}$.

    For a language $L_i$, we define a corresponding reward function $r_i$ as follows:

    \[r_i(x_s, y) = \mathbf{1}[y \in L_i \setminus s].\]

    Are reward function class of interest is thus $\calR = r_1, r_2, ...$, where each $r_i$ is identified with $L_i \in \calL$.

    We now describe how to use a learner for the above problem to generate with bounded mistakes. At the beginning of each step $t$, the elements that have been shown thus far by the adversary are some $u_1, ..., u_{t - 1}$. 

    To leverage the LfD algorithm, we act as the adversary and show the learner the context $x_{u_{1:t-1}}$. The learner outputs some $\hat y_t \in \mathcal{U}$, and we generate exactly this element in the language generation setting. We then see the adversary-generated element $u_t$, and make this the adversarial demonstration $y_t = u_t$ in the LfD algorithm. 

    Note that by definition, the noise of our language generation adversary on a particular language coincides with the sub-optimality of the demonstrator on the corresponding reward function, i.e. for any $L_i \in \calL$ and $T \in \mathbb{N}$, 

    \[\sum_{t = 1}^T \mathbf{1}[u_t \not\in L_i \setminus u_{1:t - 1}] = \sum_{t = 1}^T (1 - r_i(x_{u_{1:t-1}}, y_t). \]

    Similarly, the mistakes of our generator on $L_i$ are equal to the sub-optimality of the learner on $r_i$:
    \[\sum_{t = 1}^T \mathbf{1}[\hat{u}_t \not\in L_i \setminus u_{1:t - 1}] = \sum_{t = 1}^T (1 - r_i(x_{u_{1:t-1}}, \hat{y}_t). \]

    Thus, we conclude that if the LfD algorithm guarantees that at a particular step $T$ that for any $i$, the learner's actions satisfy the guarantee
     \[\sum_{t = 1}^T (1 - r_i(x_t, \hat y_t)) \leq M(i, T, \sum_{t = 1}^T (1 - r_i(x_t, y_t))),\]

     then the outputs of our constructed generator must also satisfy 
     \[\sum_{t = 1}^T \mathbf{1}[\hat u_t \not\in L_i \setminus u_{1:t-1}] \leq M(i, T, \sum_{t = 1}^T \mathbf{1}[u_t \not\in L_i \setminus u_{1:t-1}]).\]
    
\end{proof}

\subsection{Learning from Sub-Optimal Demonstrators and Noisy Adversaries}\label{sec:lfd-noisy}

In this section, we describe how to derive our mistake bounds for noisy adversaries from Lemmas~\ref{lem:noisy-finite} and \ref{lem:noisy-infinite}. Both results can be derived by leveraging the reduction outlined in Lemma~\ref{lem:reduction} along with an appropriate result from the LfD setting. The finite case immediately follows from an existing result of~\cite{joshi2025learning}, while Lemma~\ref{lem:noisy-infinite} follows from proper instantiation of Theorem~\ref{thm:lfd-inf-stream-noisy}.  

\subsubsection{Deriving Lemma~\ref{lem:noisy-finite}}

\cite{joshi2025learning} give the following result for learning from potentially sub-optimal demonstrations for a finite class of reward functions. 

\begin{restatable}[Theorem 4, \cite{joshi2025learning}]{thm}{joshiagnosticguarantee}
    There exists an algorithm such that for any finite $\calR$ such that $\max_{y} r(x, y) = 1$ for all $r \in \calR$ and $x \in \calX$, and any online sequence $(x_t, y_t)$, after the end of $T$ rounds, for any $0 < \gamma \leq 3/4$, 

    \[\sum_{t = 1}^T (1 - r(x_t, \hat y_t)) \leq (1 + 2\gamma)(\sum_{t = 1}^T(1 - r(x_t, y_t)) + \log_2|\calR|/\gamma.\]
\end{restatable}

Combining this result with the reduction in Lemma~\ref{lem:reduction} immediately gives the mistake bound for noisy adversaries stated in Lemma~\ref{lem:noisy-finite}.

\subsubsection{Proof of Lemma~\ref{lem:noisy-infinite}}

\noisyInf*

\begin{proof}[Proof of Lemma~\ref{lem:noisy-infinite}]
    We instantiate Algorithm~\ref{alg:inf-reward-stream} with polynomially decreasing weights $w_0(i) = 1/i^2$ and exponential growth function $f(t) = 2^t$. Note that this implies $f^{-1}(i) = \lfloor \log_2 i\rfloor$. 

    Plugging this in to the guarantee of Theorem~\ref{thm:lfd-inf-stream-noisy}, we derive the following guarantee for general binary rewards. Note that because all languages have infinite support, we can assume that $\sup_{y \in \calY} r_i(x, y) = 1$ for all $i \in \N$ and $x \in \calX$. Using the same reasoning as in Theorem~\ref{thm:non-uniform-mistake-bound}, this simplifies to

    \[\sum_{t = 1}^T (1 - r_i(x_t, \hat{y}_t)) \leq \frac{\log_2\pi^2/6}{\gamma} + (1 + 2\gamma)\sum_{t = \lfloor \log_2 i\rfloor}^{T}(1 - r_i(x_t,y_t)) + (1 + 2/\gamma)\log_2 i\]

    Substituting this into our reduction (Lemma~\ref{lem:reduction}),  we conclude that we can construct a generator with the following guarantee for every $T$ and $i$:

    \[\sum_{t = 1}^T \mathbf{1}[\hat x_t \not\in L_i] \leq (1 + 2\gamma)\sum_{t = \lfloor \log_2 i\rfloor}^T \mathbf{1}[x_t \not\in L_i] + (1 + 2/\gamma)\log_2 i,\]

    Thus matching the promised guarantee.

\end{proof}

\subsection{Learning from Demonstrations with Infinite Reward Functions}\label{sec:lfd-inf-rewards}
While we can directly leverage the results of \cite{joshi2025learning} to derive mistake bounds for finite language classes via our reduction in Lemma~\ref{lem:reduction}, the problem of language generation is particularly concerned with settings in which we must learn to generate well from a \emph{countably infinite stream} of languages. In this section, we extend the results of \cite{joshi2025learning} and give an algorithm that provides non-uniform guarantees for an infinite stream of reward functions, $\calR = r_1, r_2, \dots$. 

\begin{algorithm}[H]
    \DontPrintSemicolon
    \SetAlgoNoLine
    \caption{Learning from Arbitrary Demonstrators with an Infinite Stream of Reward Functions.}
    \label{alg:inf-reward-stream}
    
    \KwIn{Context space $\mathcal{X}$, label space $\mathcal{Y}$, infinite stream of reward functions $\mathcal{R} = \{r_1, r_2, \dots\}$ (with $r: \mathcal{X} \times \mathcal{Y} \rightarrow [0, 1]$), weights $w_0:\mathbb{N} \rightarrow \mathbb{R}_{\geq 0}$, growth function $f: \mathbb{N} \rightarrow \mathbb{N}$, parameter $\gamma \in (0, 3/4] \cup \{1\}$}
    
    \For{$t = 1, 2, \dots, \infty$}{
        Observe $x_t$\;\\
        Output $\hat y_t := \arg\max_{y \in \mathcal{Y}}\sum_{i = 1}^{f(t)}w_{t-1}(i)r_i(x_t, y)$\;\\
        Observe adversary's $y_t$\;\\
        
        \BlankLine
        \tcp{Update weights for $i \in [f(t)]$}
        \uIf{$\gamma \in (0, 3/4]$}{
            $w_t(i) \leftarrow w_{t - 1}(i)(1 + \gamma)^{\lambda_i(x_t, \hat y_t)}(1 - \gamma)^{\lambda_i(x_t, y_t)}$\;\\
            where $\lambda_i(x, y) := \sup_{y' \in \mathcal{Y}}r_i(x, y') - r_i(x, y)$\;
        }
        \ElseIf{$\gamma = 1$}{
            $w_t(i) \leftarrow \begin{cases} 
                0 & \text{if } r_i(x_t, y_t) = 0 \\ 
                w_{t-1}(i) & \text{if } r_i(x_t, y_t) = r_i(x_t, \hat y_t) = 1 \\ 
                2 w_{t-1}(i) & \text{if } r_i(x_t, y_t) = 1, r_i(x_t, \hat y_t) = 0 
            \end{cases}$\;
        }

        \BlankLine
        \tcp{Initialize weights for $i \in \{f(t) + 1, \dots, f(t + 1)\}$}
        \uIf{$\gamma \in (0, 3/4]$}{
            $w_t(i) = w_0(i)$\;
        }
        \ElseIf{$\gamma = 1$}{
            $w_t(i) = \begin{cases} 
                w_0(i) & \text{if } r(x_j, y_j) = 1 \text{ for all } j \in [t] \\ 
                0 & \text{otherwise} 
            \end{cases}$\;
        }
    }
\end{algorithm}

It will be helpful to define the following notation. For any $i \in \N$ and $T \in \N$, denote 

\[\hat M_{i, alg}^T := \sum_{t = 1}^T r_i(x_t, \hat y_t), \quad M_{i, opt}^T := \sum_{t = 1}^T \sup_{y \in\calY} r_i(x_t, y), \quad R_{i}(T':T) = \sum_{t = T'}^T\left(\sup_{y \in\calY} r_i(x_t, y) - r_i(x_t, y_t)\right).\]

Note that in the binary rewards case, if we assume that every context has at least one high-reward label, $M_{i, opt}^T$ is always equal to $T$. 

We present the following two results about Algorithm~\ref{alg:inf-reward-stream}, one for the optimal demonstrator and binary reward case, and the other for sub-optimal demonstrators. 

\begin{restatable}{thm}{LfDinfstream}\label{thm:lfd-inf-stream}
    For any infinite stream $\calR = r_1, r_2, ...$ of binary reward functions, and any online sequence $(x_t, y_t)$ satisfying $r_i(x_t, y_t) = 1$ for some $i \in \N$ and all $t \in \N$, initial weight function $w_0: \N \rightarrow \R_{\geq 0}$ satisfying $\sum_{i = 1}^{\infty}w_0(i) \leq W < \infty$, non-decreasing growth function $f:\N \rightarrow \N$, after the end of $T$ rounds, Algorithm~\ref{alg:inf-reward-stream} with input $\gamma = 1$ guarantees that for any $r_i$ such that $r_i(x_t, y_t) = 1$ for all $t \in [T]$,
    \[\sum_{t = 1}^{\infty} (1 - r_i(x_t, \hat y_t)) \leq \log_2(W/w_0(i)) + f^{-1}(i). \]
\end{restatable}

\begin{restatable}{thm}{LfDinfstreamNoisy}\label{thm:lfd-inf-stream-noisy}
    For any infinite stream $\calR = r_1, r_2, ...$, any online sequence $(x_t, y_t)$, initial weight function $w_0: \N \rightarrow \R_{\geq 0}$ satisfying $\sum_{i = 1}^{\infty}w_0(i) \leq W < \infty$, non-decreasing growth function $f:\N \rightarrow \N$, and $\gamma \in (0, 3/4]$, after the end of $T$ rounds, Algorithm~\ref{alg:inf-reward-stream} guarantees that 
    \[M_{i, opt}^T - \hat{M}_{i, alg}^T \leq \log_2(W/w_0(i))/\gamma + (1 + 2\gamma)R_i(f^{-1}(i) + 1: T) + f^{-1}(i). \]
\end{restatable}

The proof of Theorem~\ref{thm:lfd-inf-stream} follows a similar potential argument to that of Theorem~\ref{thm:general-recipe}. We begin by bounding the growth of the total weight using the following lemma:

\begin{restatable}{lem}{weightfnBdLfD}\label{lem:lfd_weight_fn_bound}
    Define $W_t := \sum_{i = 1}^{f(t+1)}w_t(i)$. Then, for any step $t \geq 1$, 
    \[W_t \leq W_{t - 1} + \sum_{i = f(t) + 1}^{f(t+1)} w_0(i).\]
\end{restatable}

\begin{proof}
    We re-express $W_t$ by decomposing the weights:
    \begin{align*}
        W_t &= \sum_{i = 1}^{f(t)} w_t(i) + \sum_{i = f(t) + 1}^{f(t + 1)} w_t(i) \\
        &\leq \sum_{i = 1}^{f(t)} w_t(i) + \sum_{i = f(t) + 1}^{f(t + 1)} w_0(i) \tag{Definition of weight initialization}
    \end{align*}

    Thus, it suffices to prove that $\sum_{i = 1}^{f(t)} w_t(i) \leq W_{t - 1} = \sum_{i = 1}^{f(t)} w_{t-1}(i) $. 

    We go about this by showing that $\sum_{i = 1}^{f(t)} w_t(i) - W_{t - 1} \leq 0$:

    \begin{align*}
        &\sum_{i = 1}^{f(t)} \left(w_t(i) - w_{t - 1}(i)\right)\\
        &=  \sum_{i = 1}^{f(t)} \left(w_{t - 1}(i)(1 + \gamma)^{\lambda_i(x_t, \hat y_t)}(1 - \gamma)^{\lambda_i(x_t, y_t)} - w_{t - 1}(i)\right)\\
        &\leq \sum_{i = 1}^{f(t)} \left(w_{t - 1}(i)(1 + \gamma\lambda_i(x_t, \hat y_t))(1 - \gamma\lambda_i(x_t, y_t)) - w_{t - 1}(i)\right)
    \end{align*}

    where the upper bound follows from the fact that $(1 +\gamma)^u \leq 1 + \gamma u$ and $(1 - \gamma)^u \leq 1 - \gamma u$ for any $u \in [0, 1]$ and $\gamma \in [0, 1]$. Note that because the range of rewards is $[0, 1]$, $\lambda_i(x, y) \in [0, 1]$ as well. 

    Simplifying the equation, we can equivalently rewrite as 
    
    \begin{align*}
        &\gamma\sum_{i = 1}^{f(t)} w_{t - 1}(i)\left(\lambda_i(x_t, \hat y_t) - \lambda_i(x_t, y_t)\right) - \gamma^2\sum_{i = 1}^{f(t)} w_{t - 1}(i)\lambda_i(x_t, \hat y_t)\lambda_i(x_t, y_t))\\
        &= \gamma\sum_{i = 1}^{f(t)} w_{t - 1}(i)\left(r_i(x_t, y_t) - r_i(x_t, \hat{y}_t)\right) - \gamma^2\sum_{i = 1}^{f(t)} w_{t - 1}(i)\lambda_i(x_t, \hat y_t)\lambda_i(x_t, y_t))\\
    \end{align*}

    By definition of our choice of $\hat y_t$, the first term is upper bounded by 0, and because all weights and $\gamma$ values are non-negative, the second term is also upper-bounded by 0, and we conclude that the difference is at most 0, proving the lemma. 
\end{proof}

We now proceed to prove Theorem~\ref{thm:lfd-inf-stream-noisy}.

\begin{proof}[Proof of Theorem~\ref{thm:lfd-inf-stream-noisy}]
        We first note that unrolling Lemma~\ref{lem:lfd_weight_fn_bound} gives a constant upper bound on any $W_t$:
        \[W_t \leq W_{t - 1} + \sum_{i = f(t) +1}^{f(t+ 1)} w_0(i) \leq \sum_{i = 1} \leq \sum_{i = 1}^{f(t + 1)} w_0(i) \leq W.\]

        For a particular reward function $r_i$, the first step where $r_i$ is considered by the algorithm is at $f^{-1}(i) + 1$. 
        We focus on how much additional regret the algorithm can incur after first introducing $r_i$. 
        
        At any timestep $T > f^{-1}(i)$, suppose the algorithm has accumulated $\hat{R}_i(T)$ regret compared to the reward-maximizing labels for $r_i$ since initially being considered by the algorithm at $t = f^{-1}(i) + 1$. Similarly denote $R_i(T)$ as the regret suffered by the demonstrator after timestep $f^{-1}(i)$. By definition of our algorithm, the weight on $r_i$ is equal to 

        \[w_T(i) = w_0(i)(1 + \gamma)^{\hat{R}_i(T)}(1 - \gamma)^{R_i(f^{-1}(i) + 1: T)} \leq W_T \leq W.\]

        Where we get an upper bound of $W$ from Lemma~\ref{lem:lfd_weight_fn_bound}. Solving for $\hat{R}_i(T)$, we conclude that

        \[\hat{R}_i(T) \leq \log_2(W/w_0(i))/\log_2(1 + \gamma) + R_i(f^{-1}(i) + 1: T)\left(\log_2(\frac{1}{1 - \gamma})/\log_2(1 + \gamma)\right).\]

        Noting that for $\gamma \in (0, 3/4]$, we have $\frac{1}{\log_2(1 + \gamma)}\leq 1/\gamma$ and $\frac{\log_2(1/(1 - \gamma))}{\log_2(1 + \gamma)}\leq 1 + 2\gamma$, and thus we conclude that

        \[\hat{R}_i(T) \leq \log_2(W/w_0(i))/\gamma + (1 + 2\gamma)R_i(f^{-1}(i) + 1: T).\]
        
        Thus we have bounded the regret after $r_i$ is first considered. We can bound the regret before $r_i$ is first considered as $f^{-1}(i)$ because the rewards line in $[0, 1]$, and thus conclude a total bound on the regret of 

        \[M_{i, opt}^T - \hat{M}_{i, alg}^T \leq \log_2(W/w_0(i))/\gamma + (1 + 2\gamma)R_i(f^{-1}(i) + 1: T) + f^{-1}(i). \]
\end{proof}

We now prove the optimal demonstrator case for binary rewards:

\begin{proof}[Proof of Theorem~\ref{thm:lfd-inf-stream}]
        We first note that unrolling Lemma~\ref{lem:lfd_weight_fn_bound} gives a constant upper bound on any $W_t$:
        \[W_t \leq W_{t - 1} + \sum_{i = f(t) +1}^{f(t+ 1)} w_0(i) \leq \sum_{i = 1} \leq \sum_{i = 1}^{f(t + 1)} w_0(i) \leq W.\]

        For a particular reward function $r_i$, the first step where $r_i$ is considered by the algorithm is at $f^{-1}(i) + 1$. 
        We focus on how much additional regret the algorithm can incur after first introducing $r_i$. 
        
        At any timestep $T > f^{-1}(i)$, suppose there have been $M$ steps where the algorithm outputted a $\hat y_t$ with 0 reward on $r_i$ since initially being considered by the algorithm at $t = f^{-1}(i) + 1$. By definition of our algorithm, the weight on $r_i$ is equal to 

        \[w_T(i) = w_0(i)2^M \leq W_T \leq W.\]

        Where we get an upper bound of $W$ from Lemma~\ref{lem:lfd_weight_fn_bound}. Solving for $M$, we conclude that

        \[M \leq \log_2(W/w_0(i)).\]

        Thus we have bounded the number of mistakes after $r_i$ is first considered. We can bound the number of mistakes before $r_i$ is first considered as $f^{-1}(i)$, and thus conclude a total bound on the number of mistakes of  

        \[\sum_{t = 1}^T (1 - r_i(x_t, \hat{y}_t)) \leq \log_2(W/w_0(i)) + f^{-1}(i).\]

        We conclude that this bound holds for all such $T$, as the right-hand-side has no dependence on $T$. 
\end{proof}

\end{document}